\pdfoutput=1

\documentclass[11pt]{article}

\usepackage[final]{acl}

\usepackage{times}
\usepackage{latexsym}
\usepackage{paralist}
\usepackage{stmaryrd}
\usepackage{booktabs}

\usepackage[T1]{fontenc}

\usepackage[utf8]{inputenc}

\usepackage{microtype}

\usepackage{inconsolata}

\usepackage{amsfonts}
\usepackage{amsmath}
\usepackage{amsthm}
\usepackage{xspace}

\usepackage{tikz}
\usetikzlibrary{fit,backgrounds} 

\newcommand{\modelname}{Toolken+\xspace}

\def\RR{\mathbb{R}}
\def\x{\mathbf{x}}
\def\h{\mathbf{h}}

\def\cV{\mathcal{V}}
\def\cT{\mathcal{T}}
\def\Wv{\mathbf{W}_{\cV}}
\def\Wt{\mathbf{W}_{\cT}}
\def\Wtt{\mathbf{W}_{\cT'}}
\def\rejtool{\textsc{Rej}}

\newcommand{\pc}[2]{p\left(#1\middle|#2\right)}
\newcommand{\pctool}[2]{p_{\mathrm{aug}}\left(#1\middle|#2\right)}
\newcommand{\pcint}[2]{p_{\mathrm{rank}}\left(#1\middle|#2\right)}
\newcommand{\pcllm}[2]{p_{\mathrm{LLM}}\left(#1\middle|#2\right)}
\def\mtk{\mathbf{m}\left(\cT_k\right)}
\newcommand{\softmax}[1]{\mathrm{softmax}\left(#1\right)}
\def\pcnext{\pc{\x_{i+1}}{\x_{\le i}}}
\newcommand\indic[1]{\left\llbracket #1 \right\rrbracket}

\usepackage[ruled,noend]{algorithm2e}

\newtheorem{theorem}{Proposition}

\title{\modelname: Improving LLM Tool Usage with Reranking and a Reject Option}

\author{
 \textbf{Konstantin Yakovlev\thanks{Work was done while at Huawei Noah’s Ark Lab}\textsuperscript{1}},
 \textbf{Sergey Nikolenko\textsuperscript{2,3}},
 \textbf{Andrey Bout\textsuperscript{*4}}
\\
\\
 \textsuperscript{1}HSE University, Moscow, Russia\\
 \textsuperscript{2}ISP RAS Research Center for Trusted Artificial Intelligence, Moscow, Russia,\\
 \textsuperscript{3}St. Petersburg Department of the Steklov Institute of Mathematics, Russia,\\
\textsuperscript{4}Yandex, Moscow, Russia\\
 \small{
   \textbf{Correspondence:} \href{mailto:kdyakovlev@hse.ru}{kdyakovlev@hse.ru}, \href{mailto:sergey@logic.pdmi.ras.ru}{sergey@logic.pdmi.ras.ru}, \href{mailto:andrey-bout@yandex-team.ru}{andrey-bout@yandex-team.ru}
 }
}

\begin{document}
\maketitle

\begin{abstract}
The recently proposed ToolkenGPT tool learning paradigm demonstrates promising performance but suffers from two major issues: first, it cannot benefit from tool documentation, and second, it often makes mistakes in whether to use a tool at all. We introduce \modelname that mitigates the first problem by reranking top $k$ tools selected by ToolkenGPT and the second problem with a special ``Reject'' option such that the model will generate a vocabulary token if ``Reject'' is ranked first. We demonstrate the effectiveness of \modelname on multistep numerical reasoning and tool selection tasks.
\end{abstract}

\section{Introduction}

Recently, large language models (LLM) have been extended by allowing access to external tools such as symbolic computation engines~\cite{gou2023tora}, databases that serve as external memory~\cite{mu-etal-2023-augmenting}, and others~\cite{Schick2023ToolformerLM, qin2023toolllm}. Tool learning paradigms can be broadly divided into
\begin{inparaenum}[(1)]
\item \emph{supervised fine-tuning} to leverage tools \cite{Schick2023ToolformerLM}, 
which works well but lacks flexibility and cannot generalize to unseen tools, and
\item \emph{in-context learning}~\cite{lu2023chameleon}, where demonstrations are provided in the prompt; this method is very easy to extend and generalize but often bumps against inherent context length limitations of LLMs~\cite{tay2021long}.
\end{inparaenum}
%
\emph{ToolkenGPT} \cite{hao2023toolkengpt} aims to have the best of both worlds:
each tool is represented by a special token called \emph{toolken} that has a trainable embedding and extends the vocabulary. Once a toolken has been predicted, the model switches into ``tool mode'' where it uses in-context examples to fill in the tool's arguments, calls it, sends the results back to text, and returns to language modeling mode.
Toolkens require very few parameters (toolken embeddings) to be trained while keeping LLM weights frozen, and there is no limit on the amount of data to train these parameters.

\begin{figure}[!t]
\includegraphics[width=\linewidth]{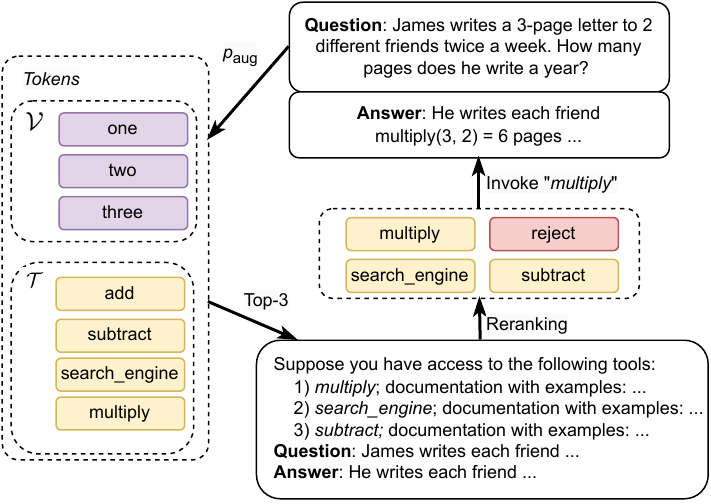}

\caption{\modelname sample operation.}\label{fig:toolken}
\end{figure}

In this work, we extend ToolkenGPT with two novel features that aim to fix
two important issues. First, ToolkenGPT cannot use tool documentation known to be helpful for LLMs~\cite{hsieh2023tool}; we will show that ToolkenGPT is often unsure which tool to use, and documentation could help decide this. To this end, we introduce a copy of toolken embeddings that rerank retrieved tools, i.e., take top-$k$ tool candidates, prepend the prompt with their documentations, and ask the LLM to choose the most relevant tool. Second, ToolkenGPT often makes mistakes in judging when to use tools, calling them too often. To alleviate this, we introduce an extra $\rejtool$ (``\emph{Reject}'') ``tool'' that switches back to text generation without invoking any tools. $\rejtool$ is provided as an option for the reranking mechanism introduced above. The entire operation of \modelname is illustrated in Figure~\ref{fig:toolken}. 

Overall, we aim to minimize false positive errors for tool invocations and tool misclassification rate for ToolkenGPT. This significantly improves the model's robustness, allowing for developing more trusted LLM agents that can have access to a wider variety of tools.
%
A general tool usage paradigm for LLMs~\cite{yang2023gpt4tools, huang2023metatool} has four stages: whether to use a tool, which tool to use, infilling the arguments, and dealing with the tool's output. Our methods improve the first and second stages of this process.
Moreover, we provide a formal justification for the toolken training algorithm based on variational inference.
We evaluate our results on the GSM8K \cite{cobbe2021training}, MetaTool \cite{huang2023metatool}, and VirtualHome \cite{Puig_2018_CVPR} datasets, showing significant improvements.

Thus, our contributions are as follows:
\begin{inparaenum}[(1)]
    \item solutions to issues associated with the first two stages of the tool usage process for LLMs;
    \item a theoretically grounded training procedure for the introduced toolken embeddings;
    \item an empirical evaluation study supporting the efficiency of our approach.
\end{inparaenum}
Below, Section~\ref{sec:related} compares our approach with recent related work, Section~\ref{sec:method} introduces our modifications and formal justification for training and inference, Section~\ref{sec:eval} presents experimental results, and Section~\ref{sec:concl} concludes the paper.

\section{Related Work}\label{sec:related}

\textbf{Quantifying the uncertainty of LLMs}. Recently, \citet{zhang2023r} introduced a tuning method to teach LLMs to refrain from answering the question if the LLM is not sure in its answer, reducing hallucinations and improving uncertainty estimation. In contrast to this study, our approach does not require any fine-tuning of the LLM, instead we learn an embedding corresponding to the rejection tool. In another line of research, \citet{diao2023active} proposed active prompting focusing on finding the best task-specific prompt. In this work, we focus on task-independent prompts.

\textbf{Natural language feedback}. \citet{huang-etal-2023-large} suggested the idea to use LLM feedback as training inputs, using the generated rationale-augmented answers to fine-tune the model. To alleviate incorrect reasoning steps with tool usage, \citet{paul2023refiner} suggested to generate natural language feedback from a critic model learned separately. Note that although this work also addresses the issue of using incorrect operations in the Math World Problem task, it requires to train a critic model. 

Overall, our approach adheres to the paradigm of prompted refiners as shown in \cite{madaan2023self, shinn2023reflexion, gou2023critic}, where the same frozen model is used for reasoning and providing feedback. Our method does not invoke the tools themselves and does not work with their outputs, only produces special tokens for tool invocation.
In a recent work, \citet{an2024learning} also used model mistake correction to improve the quality of solving math problems. However, their suggested approach relied on GPT-4 output and, again, needed to fine-tune the model.

\textbf{Chain of thought reasoning}. 
Chain of thought prompting was originally introduced by \citet{wei2023chainofthought, zhou2022least} to enhance the ability of large language models to perform complex reasoning. \citet{NEURIPS2022_8bb0d291} showed that a LLM is a good zero-shot reasoner with a simple prompt before each answer.
To further improve reasoning skills, \citet{wang2022self} introduced self-consistency decoding that reranks the generated rationales by taking a majority vote over the final numerical answers. Similarly, our approach also benefits from reranking that can use additional information such as tool documentation. 

Recently, \citet{zelikman2022star} proposed a bootstrapping technique that was able to improve the performance on reasoning tasks even without a massive rationale dataset. In their approach, the model is repeatedly fine-tuned on a dataset of self-generated rationales; our approach is similar to this one since bootstrapping is used to train the introduced toolkens but we do not need to tune the LLM weights.

\textbf{Tool-augmented language models}. One direction for augmenting LLMs with external tools is fine-tuning for tool use \cite{qin2023toolllm, liang2023taskmatrix, Schick2023ToolformerLM, patil2023gorilla}; these methods achieve excellent performance but suffer from poor adaptability to unseen tools and high computational requirements to fine-tune the LLM. 

Another paradigm learns tool use in context, from documentation and/or demonstrations added to the LLM input \cite{lu2023chameleon, paranjape2023art, shen2023hugginggpt}; these approaches do not require fine-tuning and can learn a new tool given a handful of demonstrations but suffer from performance degradation because of limited context length. 

In this work, we propose an improvement for the Toolken paradigm recently proposed by \citet{hao2023toolkengpt} that aims to take the best of both worlds. It introduces learnable tool embeddings (toolkens) trained on a dataset of tool use examples while keeping the weights of the LLM frozen; it also leverages in-context learning to fill in tool arguments.

\section{Method}\label{sec:method}

\textbf{ToolkenGPT}.
ToolkenGPT introduces an embedding for each tool concatenated with the language modeling head. Formally, the embedding matrix $\Wv \in \RR^{|\cV| \times d}$,
where $\cV$ is the original vocabulary of tokens and $d$ is the latent dimension, is extended with a matrix $\Wt\in \RR^{|\cT| \times d}$ for a set of tools $\cT = \{t_1, \ldots, t_{|\cT|}\}$, and the next token probability for the augmented LLM is calculated as
 $\pctool{\x_i}{\x_{<i}} = \softmax{\left[\Wv, \Wt\right]\h_{i-1}}$.
Inference is divided into two interleaving stages: reasoning mode, where the model generates a rationale using $\pctool{\x_i}{\x_{<i}}$, and tool mode, where it infills tool arguments given a prompt with usage examples. Thus, $\Wt$ are learned by solving
\begin{equation}\label{eq:opt2}\small
\min_{\Wt}\sum\nolimits_{X\in D}\sum\nolimits_{i=1}^{|X|}-\log\pctool{\x_{i+1}}{\x_{\le i}}.
\end{equation}


\textbf{\modelname}.
First, we extend the tool set as $\cT' = \cT \cup \{\rejtool\}$, where $\rejtool$ is a special tool responsible for switching back to reasoning mode. Second, instead of taking the best proposed tool we take the retrieved subset of top $k$ tools $\cT_k\subseteq\cT'$ and ask the model to choose one. The proposed \modelname model takes the previously generated sequence and top $k$ tools retrieved  by ToolkenGPT $\cT_k$ as input and is asked to generate a tool from $\cT_k\cup\{\rejtool\}$. Formally, \modelname produces $\pcint{\x_i}{\x_{<i}, \cT_k} = \softmax{\Wtt\h_{i-1} + \mtk}$, where $\mtk\in\RR^{|\cT'|}$ is the mask vector with $\mtk_t=0$ for $t\in\cT_k\cup\{\rejtool\}$ and $-\infty$ otherwise.
%
\modelname's inference procedure is shown in Algorithm~\ref{alg:inf}. 


\textbf{Approximate inference}.
The next token probabilities in Algorithm \ref{alg:inf} are given by

\noindent{\small
\begin{multline*}
\pcnext = \indic{\x_{i+1}\in\cV}\left(\pctool{\x_{i+1}}{\x_{\le i}} \right. \\
 + \left.\pctool{\cT}{\x_{\le i}}\pcint{\rejtool}{\x_{\le i}, \cT_k}\pcllm{\x_{i+1}}{\x_{\le i}}\right) \\
  + \indic{\x_{i+1}\in\cT}\pctool{\cT}{\x_{\le i}}\pcint{\x_{i+1}}{\x_{\le i}, \cT_k},
\end{multline*}}

\noindent
where $\indic{\cdot}$ is the indicator. For a dataset $D=\{X_n\}_{n=1}^N$, tool embeddings are found by solving
\begin{equation}\label{eq:opt1}\small
\min_{\Wt,\Wtt}\sum\nolimits_{X\in D}\sum\nolimits_{i=1}^{|X|}-\log\pcnext,
\end{equation}
which is non-differentiable w.r.t. $\Wt$, so we optimize the original ToolkenGPT model with its own criterion \eqref{eq:opt2}
and then optimize~\eqref{eq:opt1} w.r.t. $\Wtt$ only. Problem~\eqref{eq:opt1}, however, suffers from computational instabilities caused by the product of model probabilities in $\pcnext$, so we propose to optimize a computationally stable upper bound instead.

\begin{theorem}[Naive upper bound]
The following is a computationally stable upper bound of \eqref{eq:opt1}, up to an additive constant independent of $\Wtt$:

{\noindent\small
\begin{multline}\label{eq:opt3}
    \sum\nolimits_{X\in D}\sum\nolimits_{i=1}^{|X|} 
    \left( -\indic{\x_{i+1}\in\cV}\log\pcint{\rejtool}{\x_{\le i}, \cT_k} \right. \\
    \left. -\indic{\x_{i+1}\in\cT}\log\pcint{\x_{i+1}}{\x_{\le i}, \cT_k} \right).
\end{multline}
}
\end{theorem}

\begin{proof}
We transform (omitting $\cT_k$ for brevity)

{\noindent\small
\begin{multline*}
\log\pcnext = \log\pcnext\left(\indic{\x_{i+1}\in\cV} + \right. \\
\left. + \indic{\x_{i+1}\in\cT}\right) \ge \indic{\x_{i+1}\in\cV}\left(\log\pcint{\rejtool}{\x_{\le i}} + \right. \\ \left. + \log \pctool{\cT}{\x_{\le i}} + \log\pcllm{\x_{i+1}}{\x_{\le i}} \right) + \\
+ \indic{\x_{i+1}\in \cT}(\log \pctool{\cT}{\x_{\le i}} + \log \pcint{\x_{i+1}}{\x_{\le i}}),
\end{multline*}
}

\noindent
where the inequality holds because $\pctool{\x_{i+1}}{\x_{\le i}}$ is always nonnegative, and then obtain~\eqref{eq:opt3} by removing the terms independent of $\Wtt$.
\end{proof}




\noindent
Bound~\eqref{eq:opt3} is differentiable but still computationally hard: it requires $|X|$ forward passes to find the loss for a single data point, retrieving top $k$ tools for every $i$. Therefore, we propose a simplified objective that uses only $i$ where ToolkenGPT erroneously predicts a toolken instead of a regular token:

{\noindent\small
\begin{multline}\label{eq:opt5}
\min_{\Wtt}\sum_{X\in D}\sum_{i=1}^{|X|}\left(
\vphantom{\begin{matrix}\x_{i+1}\in\cV \\ \arg\max p_{\mathrm{t}}\in \cT\end{matrix}}
-\indic{\x_{i+1}\in\cT}\log\pcint{\x_{i+1}}{\x_{\le i}, \cT_k} \right. \\ \left.
 -\indic{\begin{matrix}\x_{i+1}\in\cV \\ \arg\max p_{\mathrm{t}}\in \cT\end{matrix}}
 \log\pcint{\rejtool}{\x_{\le i}, \cT_k}\right).
\end{multline}}

\noindent
We train \modelname with~\eqref{eq:opt5} to correct the errors of ToolkenGPT or rescore its outputs. Objective \eqref{eq:opt5} is not guaranteed to be an upper bound of~\eqref{eq:opt2} but can be viewed as an approximation via hard negative mining (``hard'' prefixes $\x_{\leq i}$ are those where ToolkenGPT incorrectly predicts tool use).

\begin{algorithm}[!t]\small
\caption{\modelname inference}\label{alg:inf}
\KwData{$\pcllm{\x_i}{\x_{<i}}$, $\pctool{\x_i}{\x_{<i}}$, $\pcint{\x_i}{\x_{<i}, \cT_k}$, user query $q$}
\KwResult{Rationale $\x_{1:n}$ (with a user query)}
$x \gets q$, $i \gets |q|$\;
\While{$\x_i \not= \textrm{EOS}$}{
$\x_{i+1}^{(0)} \sim \pctool{\cdot}{\x_{\leq i}}$\;
\eIf{$\x_{i+1}^{(0)}\in \cT$}{
    $\cT_k \gets \mathrm{TopkTools}(\pctool{\cdot}{\x_{\leq i}})$\;
    $\x_{i+1}^{(1)} \sim \pcint{\cdot}{\x_{\leq i}, \cT_k}$\;
    \leIf{$\x_{i+1}^{(1)} = \rejtool$}{
        $\x_{i + 1} \sim \pcllm{\cdot}{\x_{\leq i}}$
    }{
        $\x_{i + 1} \gets \x_{i + 1}^{(1)}$
    }
}{
    $\x_{i + 1} \gets \x_{i + 1}^{(0)}$\;
}
$i \gets i + 1$\;
}
\end{algorithm}

\section{Experiments}\label{sec:eval}

\textbf{Datasets and setup}.
We evaluate \modelname on three datasets.
\emph{GSM8K}~\cite{cobbe2021training} is a parallel dataset of math problems and their rationales. We use the multistep reasoning task with four arithmetic operations as tools, removing equations from the rationales except for intermediate results (e.g., ``Weng earns 12/60 = 0.2 per minute'' becomes ``Weng earns 0.2 per minute'') to make tool selection harder.
\emph{MetaTool}~\cite{huang2023metatool} is a parallel dataset of user queries and tools with their descriptions; we use all available tools for tool selection.
\emph{VirtualHome}~\cite{Puig_2018_CVPR} is a dataset of complex household activities represented by plans, sequences of verb-object expressions where verbs and objects are external tools. Following \citet{hao2023toolkengpt}, we split it into a training set of 247 tasks and a test set of 50 tasks, using 25 verbs and 32 objects in total. Hyperparameter values and detailed experimental settings are reported in the \nameref{appendix}.
We use open source LLMs: Llama2-7B, Llama2-7B-chat~\cite{touvron2023llama,touvron2023llama2}, Vicuna-7B, Vicuna-13B~\cite{vicuna2023}.

The hyperparameters used to train all our models are reported in Table \ref{tab:hyperparams}. Overall, we used one data point per parameters update and used the same hyperparameters for ToolkenGPT and \modelname regardless of the LLM. All models were trained using the Adam optimizer~\cite{kingma2014adam}.

\begin{table}[!t]\centering\small
    \begin{tabular}{lcc}\toprule
        \textbf{Dataset} & \textbf{Learning rate} & \textbf{Epochs}\\ \midrule
        GSM8K & $10^{-4}$ & 5 \\
        MetaTool & $10^{-4}$ & 1-3 \\
        VirtualHome & $10^{-3}$ & 5-10 \\ \bottomrule
    \end{tabular}
    \caption{Hyperparameters used for training ToolkenGPT and \modelname.}
    \label{tab:hyperparams}
\end{table}

For \emph{MetaTool}, we took single-tool data that contains about 20K samples. We split it into a test split of 2K examples and two folds of about 9K examples each. To construct the training split for the GSM8K dataset, we removed all equations except for intermediate results. Moreover, we also performed the same procedure with released prompts for reasoning mode and tool mode of ToolkenGPT~\cite{hao2023toolkengpt}. Additionally, for the rejection mechanism of \modelname we used the processed prompt of the reasoning mode of ToolkenGPT. For \emph{VirtualHome}, we follow the setup of ToolkenGPT~\cite{hao2023toolkengpt} with the only difference that we split the training data into two folds. \modelname reranks top-3 retrieved objects by ToolkenGPT listed in ascending order by relevance. The prompts used for \modelname are shown in the Appendix.

\begin{table}[!t]\centering\small\setlength{\tabcolsep}{2pt}
\begin{tabular}{p{.15\linewidth}lcccc}\toprule
\textbf{LLM} & \textbf{Tool model} & \textbf{MetaTool} & \textbf{GSM8K} & \multicolumn{2}{c}{\textbf{VirtualHome}} \\
& & Rec@1 & Match & Strict & Relaxed \\\midrule
Vicuna- & 4-shot & - & 16.2 & 0.04 & 0.2 \\
7B    & ToolkenGPT & 0.623 & 16.9 & \textbf{0.62} & 0.72 \\
    & \modelname & \textbf{0.643} & \textbf{18.8} & 0.48 & \textbf{0.74} \\ \midrule
Vicuna- & 4-shot & - & 17.8 & 0.16 & 0.30 \\
13B & ToolkenGPT & 0.646 & 18.4 & 0.34 & 0.54 \\
 & \modelname & \textbf{0.662} & \textbf{19.1} & \textbf{0.58} & \textbf{0.66} \\ \midrule
Llama2- & 4-shot & - & 12.7 & 0.08 & 0.24 \\
7B-chat & ToolkenGPT & 0.642 & 11.7 & 0.44 & \textbf{0.66} \\
 & \modelname & \textbf{0.692} & \textbf{12.8} & \textbf{0.56} & \textbf{0.66} \\ \midrule
Llama2- & 4-shot & - & \textbf{10.3} & 0.18 & 0.26 \\
13B-chat & ToolkenGPT & 0.704 & 8.8 & 0.18 & 0.20 \\
 & \modelname & \textbf{0.733} & 9.6 & \textbf{0.54} & \textbf{0.58} \\ \bottomrule
\end{tabular}

\caption{Experimental results on the \emph{MetaTool}, GSM8K, and \emph{VirtualHome} datasets.}\label{tab:eval}
\end{table}


\textbf{Tool selection}.
In \emph{MetaTool}, the task is to retrieve a single tool given a query, so there is no need to use the proposed $\rejtool$ tool, and we use this task to validate the idea of reranking ToolkenGPT's outputs in isolation. Thus, we remove the first term from the loss function~\eqref{eq:opt5}.
We split the training set into two folds, train ToolkenGPT on the first fold, and train \modelname on the second fold given the top 5 outputs of ToolkenGPT. We compare both models on a held-out set and report Recall@1 in Table~\ref{tab:eval}. There is no few-shot version for \emph{MetaTool} since 200 tool descriptions do not fit into a prompt.
We see that \modelname significantly improves the results across all considered LLMs.

        

\textbf{Multistep reasoning}.
In the GSM8K dataset, the task is to generate a numerical answer to a question using the four basic arithmetic operations as tools. Here we demonstrate that the rejection mechanism improves the resulting accuracy even in isolation, so in this setup $k=1$.
We again split the data into two folds, train ToolkenGPT on one fold, collect the errors made on the second fold and train \modelname with an extended set of tools. Note that since $k=1$
the inference speed for the two models is virtually identical. Table~\ref{tab:eval} reports exact match results and shows that while both tool models improve over the vanilla LLM prompted with 4 samples from the training set, \modelname provides additional performance improvements.

\begin{table}[!t]\centering\small
    \begin{tabular}{lccccc}\toprule
    \textbf{Tool} &  & \multicolumn{3}{c}{\textbf{Recall}} & \textbf{Latency} \\
    \textbf{model} & $k$ & \textbf{@1} & \textbf{@3} & \textbf{@5} & \\ \midrule
        ToolkenGPT & 1 & 0.623 & 0.762 & 0.814 & 1.0x \\
        \modelname & 3 & 0.661 & 0.762 & 0.762 & 2.6x \\
        & 5 & 0.652 & 0.780 & 0.814 & 3.0x \\
        & 7 & 0.653 & 0.775 & 0.822 & 3.3x \\
        & 10 & 0.637 & 0.802 & 0.843 & 4.1x \\ \bottomrule
    \end{tabular}
    \caption{Ablation study on $k$, Vicuna-7B base model.}
    \label{tab:abl_topk}
\end{table}

\begin{table}[!t]\centering\small\setlength{\tabcolsep}{0pt}
    \begin{tabular}{p{\linewidth}}\toprule
        \textbf{Question}:
        Harry slept 9 hours last night. His friend James slept only 2/3 of what Harry slept. How many more hours did Harry sleep than James?
        \textbf{Answer}: 3 \\[2pt]
        \textbf{ToolkenGPT}:
         Harry slept '<multiply>(9,2/3) = 6' hours less than James slept. James slept '<divide>(9,2/3) = 13.5' hours (2/3 of 9) and Harry slept '<add>(9,6) = 15' hours (9 - 6 = 3). So Harry slept 15 hours and James slept 13.5 hours. \#\#\#\# 15 \\[2pt]
         \textbf{\modelname}:
         Harry slept '<multiply>(9,2/3) = 6' hours less than James slept. James slept '<reject>'2/3 of what Harry slept, so he slept 2/3 * 9 = 6 hours. Harry slept 9 hours, so he slept 9 - 6 = 3 hours more than James. \#\#\#\# 3 \\ \midrule
         \textbf{Question}: Geb is 10 less than half the age of Haley. If Haley is 26 years old, how old is Geb? \textbf{Answer}: 3 \\[2pt]
         \textbf{ToolkenGPT}:
         Half the age of Haley is '<divide>(26,2) = 13' years. So Geb is '<add>(10,13) = 23' years old. \#\#\#\# 23 \\[2pt]
         \textbf{\modelname}:
         Half the age of Haley is '<divide>(26,2) = 13' years. So Geb is <reject>10 less than that, which means he is '<subtract>(13,10) = 3' years old. \#\#\#\# 3 \\ \bottomrule
    \end{tabular}
    \caption{Case study on GSM8K and Vicuna-7B.}
    \label{tab:case_study}
\end{table} 

\textbf{Embodied plan generation}. In the \emph{VirtualHome} dataset, the task is to generate a sequence of actions given a question. We follow the setup of the tool selection task but now generate a contiguous sequence of actions, setting $k=3$ for \modelname
and running \modelname only on actions that correspond to an object. We compare the proposed approach with ToolkenGPT and report the success rate and its relaxed version (share of plans that pass through a target state) in Table~\ref{tab:eval}; we see that \modelname consistently improves over ToolkenGPT.


\textbf{Ablation study}.
In the ablation study, we evaluate how performance depends on the number of tools to be reranked by \modelname. Table~\ref{tab:abl_topk} shows the results evaluated on the \emph{MetaTool} dataset with the Vicuna-7B base model; we prepend tool descriptions in the order ranked by ToolkenGPT. We see that performance reaches its optimal value at $k=3$ and degrades with increasing $k$.

\textbf{Case study}. 
Table~\ref{tab:case_study} illustrates the difference in reasoning between ToolkenGPT and \modelname with specific examples from GSM8K (we used Vicuna-7B). 
The reported examples show how the rejection mechanism can allow to prevent the LLM from confusing arithmetic operations calls, which would otherwise lead to an incorrect answer.

\section{Conclusion}\label{sec:concl}

In this work, we have proposed an improvement for the ToolkenGPT approach of learning special token embeddings that adds a reject option and a reranking mechanism for tool selection. Our approach significantly improves the results via in-context learning while still keeping LLM weights frozen and learning only toolken embeddings. 

The considered approach is an important step towards improving the robustness of AI agents and user-facing tools based on modern LLMs. Better tool use not only makes the tool usage results more robust but also has a potential to reduce hallucinations and make LLM answers more trustworthy by allowing an LLM to reliably run external tools to verify its answer; this is a very important consideration in practical usage.

In future work, we hope to extend the \modelname approach to other external tools and/or agents to further expand the capabilities of modern LLMs.

\section{Limitations}\label{sec:limit}

One important limitation of this work is that the LLM used for \modelname should be aware of the tools retrieved by ToolkenGPT in the sense that its tool retrieval accuracy should be sufficiently high. The LLM also should be able to improve the ranking of these tools by reading their descriptions: \modelname relies on the accuracy of this reranking but it is out of our hands. Another limitation is that \modelname has only been evaluated on a limited number of tasks. To make the results more convincing, the framework should be tested on a wide range of tool-learning tasks and datasets. 

\section*{Acknowledgements}
This work was supported by a grant for research centers in the field of artificial intelligence, provided by the Analytical Center for the Government of the Russian Federation in accordance with the subsidy agreement (agreement identifier 000000D730321P5Q0002) and the agreement with the Ivannikov Institute for System Programming of the Russian Academy of Sciences dated November 2, 2021 No. 70-2021-00142.

\bibliography{custom}


\appendix

\section{Appendix: Prompts for \modelname}
\label{appendix}





Prompt for Vicuna-7B and Vicuna-13B:
\begin{small}
\begin{verbatim}
[System]
Below is the instruction that describes a task.
Write a response using the API tools that
appropriately completes the request.
Your output should follow this format:
Action: API call

[Question]
[QUESTION]

[The Start of Assistant's Answer]
Action: 
\end{verbatim}
\end{small}
Prompt for Llama2-7B-chat and Llama2-13B-chat

\begin{small}
\begin{verbatim}
<<SYS>>
Below is the instruction that describes a task.
Write a response using the API tools that
appropriately completes the request.
Your output should follow this format:
Action: API call
<</SYS>>

[INST]
[QUESTION]
[/INST]
Action: 
\end{verbatim}
\end{small}


Prompt for Vicuna-7B and Vicuna-13B:
\begin{small}
\begin{verbatim}
[System]
Suppose you have access to 
the following API tools:
1. tool name: [NAME], 
tool description: [DESCRIPTION],
example question: [EXAMPLE].
....
Below is the instruction that describes a task.
Write a response using the API tools that
appropriately completes the request.
Your output should follow this format:
Action:

[User Question]
[QUESTION]

[The Start of Assistant's Answer]
Action: $
\end{verbatim}
\end{small}

Prompt for Llama2-7B-chat and Llama2-13B-chat:
\begin{small}
\begin{verbatim}
<<SYS>>
Suppose you have access to 
the following API tools:
1. tool name: [NAME], 
tool description: [DESCRIPTION],
example question: [EXAMPLE].
....
Below is the instruction that describes a task.
Write a response using the API tools that
appropriately completes the request.
Your output should follow this format:

Action: API call
<</SYS>>

[INST]
[QUESTION]
[/INST]
Action: 
\end{verbatim}
\end{small}
Empirically we found that the tools selected by ToolkenGPT should appear in descending order of relevance.



\begin{small}
\begin{verbatim}
Task 1:
...
Task 4:
I am in [ROOM]. The objects I can manipulate
are [OBJECTS].
Goal:
[GOAL]
Hint:
[HINT]
Plan:
Which of the objects: <obj1>, <obj2>, <obj3>
is best to continue the plan?
[PLAN]
\end{verbatim}
\end{small}

\end{document}